\documentclass{article}

\usepackage{microtype}
\usepackage{graphicx}
\usepackage{subfigure}
\usepackage{booktabs} 

\usepackage{hyperref}

\usepackage[accepted]{icml2020}

\usepackage{amsmath}
\usepackage{amsfonts,amsthm,amssymb}
\usepackage{thmtools}
\usepackage{ifpdf}
\usepackage{listings}
\usepackage{diagbox}
\usepackage{enumitem}
\usepackage{multirow}
\usepackage{multicol}
\usepackage{color}
\usepackage{soul}
\usepackage{hhline}
\usepackage{thm-restate}
\usepackage{graphicx}
\usepackage[font=small,labelfont=bf]{caption}

\usepackage{pdfpages}

\newtheorem{theorem}{Theorem}

\newtheorem{lemma}{Lemma} 
\newtheorem{fact}{Fact} 
 
\theoremstyle{definition}

\newcommand{\OO}{\mathcal{O}}
\newcommand{\R}{\mathbb{R}}
\newcommand{\pP}{\mathbb{P}}
\newcommand{\E}{\mathbb{E}}

\newcommand{\algo}{\mathrm{ALGO}}
\newcommand{\opt}{\mathrm{OPT}}
\newcommand{\Reg}{\mathrm{Regret}}

\newcommand{\M}{\mathcal{M}}

 \DeclareMathOperator*{\argmax}{arg\,max}

\icmltitlerunning{Refined Analysis of FPL for Adversarial Markov Decision Processes}

\begin{document}

\twocolumn[
\icmltitle{
           Refined Analysis of FPL for Adversarial Markov Decision Processes}



\icmlsetsymbol{equal}{*}

\begin{icmlauthorlist}
\icmlauthor{Yuanhao Wang}{to}
\icmlauthor{Kefan Dong}{to}
\end{icmlauthorlist}

\icmlaffiliation{to}{Institute for Interdisciplinary Information Sciences, Tsinghua University}

\icmlcorrespondingauthor{Yuanhao Wang}{abrowndwarf@gmail.com}

\icmlkeywords{Machine Learning, ICML}

\vskip 0.3in
]



\printAffiliationsAndNotice{}  

\begin{abstract}
We consider the adversarial Markov Decision Process (MDP) problem, where the rewards for the MDP can be adversarially chosen, and the transition function can be either known or unknown. In both settings, Follow-the-Perturbed-Leader (FPL) based algorithms have been proposed in previous literature. However, the established regret bounds for FPL based algorithms are worse than algorithms based on mirror-descent. We improve the analysis of FPL based algorithms in both settings, matching the current best regret bounds using faster and simpler algorithms.
\end{abstract}

\section{Introduction}
\vspace{-0.1cm}
Markov Decision Processes (MDPs) are widely used to model reinforcement learning problems. Normally the reward is assumed to be stochastic and stationary, which does not capture nonstationary or adversarial environments. Recently, there is a surge of interest in studying the adversarial MDP problem~\cite{even2009online}.
There are several formulations for this problem, differing in whether the transition is known to the agent, and how the reward function is revealed. In the full information feedback setting, the reward vector is revealed to the agent at the end of each episode~\cite{even2009online, zimin2013online, neu2012adversarial, rosenberg2019online}, and in the bandit feedback setting, the agent can only observe the reward along the trajectory~\cite{rosenberg2019online, jin2019learning}. In this work, we focus on the full information feedback with both known and unknown transition. 

Roughly speaking, there are mainly two approaches to the adversarial MDPs problem, namely algorithms based on Follow-the-Perturbed-Leader (FPL)~\cite{even2009online,neu2012adversarial}, and algorithms based on mirror descent~\cite{zimin2013online,rosenberg2019online,jin2019learning}. Compared to mirror descent based algorithms, FPL has the advantage that it is conceptually simpler, easier to implement and runs faster. However, current state-of-the-art regret bounds are achieved by mirror descent based algorithms (see Table~\ref{tab:summary}).

In this work, we refine the analysis of FPL in two settings (known transition and unknown transition) by leveraging a simple observation. We show that for full information feedback adversarial MDPs, FPL-based algorithms are able to achieve the same state-of-the-art regret bounds as those of mirror descent algorithms (see Table~\ref{tab:summary}).


\renewcommand{\arraystretch}{1.5}
    \begin{table*}[htb]
        \centering
        \begin{tabular}{|c|c|c|c|}
        \hline
            & FPL based &  {FPL-refined (ours)} & Mirror descent based \\ \hline
        Known transition     &  $\Tilde{O}(\sqrt{SAT}$)~\cite{even2009online} & {$\Tilde{O}(\sqrt{T})$} & $\Tilde{O}(\sqrt{T})$~\cite{zimin2013online} \\ \hline
        Unknown transition     & $\Tilde{O}(SA\sqrt{T}$)~\cite{neu2012adversarial} & {$\Tilde{O}(S\sqrt{AT})$} & $\Tilde{O}(S\sqrt{AT})$~\cite{rosenberg2019online} \\ \hline
        \end{tabular}
        \vspace{5pt}
        \caption{Summary of regret bounds in two settings. Note that FPL-based algorithms are able to achieve the same state of the art regret bounds as those of mirror descent algorithms.}
        \label{tab:summary}
    \end{table*}


\section{Preliminaries}
\vspace{-0.1cm}
\textbf{Markov Decision Process and RL}
A finite horizon Markov Decision Process (MDP) $\M$ is defined by a five tuple  $\langle \mathcal{S}, \mathcal{A} , p, r , H \rangle$, where $\mathcal{S}$ is the state space, $\mathcal{A}$ is the action space, $p(s'|s,a)$ is the transition function, $r: \mathcal{S} \times \mathcal{A}\times H\to [0,1]$ is the deterministic reward function, and $H$ is the horizon length. Let $S=|\mathcal{S}|$ and $A=|\mathcal{A}|$ denote the number of states and the number of actions respectively.

In an episodic reinforcement learning task, the agent interacts with the environment for $T$ episodes. In the $t$-th episode, the agent starts from the initial state $s_1^t$; at state $s_h^t$, the agent chooses an action $a_h^t$, observes the reward $r_h^t$ and transits to the next state $s_{h+1}^t$. After $H$ steps, the episode ends, and the agent proceeds to the next episode.

A policy refers to a mapping from $\mathcal{S}\times [H]$ to $\mathcal{A}$. The value function of a policy $\pi$ is defined as
\begin{align*}
	V^{\pi}_h(s):=\E\left[\left.\sum_{h'=h}^{H}r(s_{h'},\pi(s_{h'},h'),h)\right|s_h=s\right].
\end{align*}
We use $\pi^*$ to denote the optimal policy, and $V^*_h$ to denote the value function of the optimal policy. The action-value function is defined as
\begin{align*}
&Q^{\pi}_h(s,a):=r(s,a)+\\&\E\left[\left.\sum_{h'=h+1}^{H}r(s_{h'},\pi(s_{h'},{h'}),,h)\right|s_h=s,a_h=a\right].
\end{align*}
Similarly, $Q^*_h$ denotes the action-value function of the optimal policy. For the standard stationary reinforcement learning, regret is defined as $\Reg(T):=\sum_{t=1}^T \left[V_1^{*}(s_1^t)-V_1^{\pi_t}(s_1^t)\right].$ The initial state can be either fixed or arbitrarily chosen~\cite{jin2018q}.

\textbf{Adversarial MDP}
In the adversarial MDP problem, the reward function $r$ for each episode can be different and is adversarially chosen. In particular, in the $t$-th episode, the reward function is $r_t$. We assume that in the end of the $t$-th episode, the complete reward function $r_t$ is revealed to the agent. We define regret as
\begin{equation*}
    \Reg(T):=\max_{\pi}\sum_{t=1}^T V_1^\pi(s_1^t,r_t)-\E\left[\sum_{t=1}^TV_1^{\pi_t}(s_1^t,r_t)\right],
\end{equation*}
where we use $V^{\pi}_h(s,r)$ to denote the value of policy $\pi$ starting from state $s$ at layer $h$ under the reward vector $r$.

Without loss of generality, we assume $s_1^{t}=s_1$ for all $1\le T\le T.$ The notation $V^{\pi}(r)$ is used as a shorthand for $V_1^{\pi}(s_1,r)$. We consider two setting for this problem. In the first setting, we assume that the transition function $p(\cdot|\cdot,\cdot)$ is known. In this case, the problem resembles more closely the expert problem.
In the second setting, the transition $p(\cdot|\cdot,\cdot)$ is unknown but fixed. In this case, the poblem resembles more closely the usual reinforcement learning problem.


\textbf{Notations} Let $r_{1:t}:=\sum_{\tau=1}^{t}r_\tau$ be the summation of the reward function from episode $1$ to $t$. For simplicity we define $(\pP f)(s,a):=\E_{s'\sim p(\cdot\mid s,a)}f(s')$. Let $\opt:=\max_{\pi}V^{\pi}(r_{1:T})$ be the total reward of the optimal policy in hindsight, and $\algo:=\E[\sum_{t=1}^{T}V^{\pi_t}(r_t)]$ be the expected total reward that the algorithm collects. By $\text{Exp}(\eta)$, we mean an exponential distribution with mean $1/\eta$. In other words, the density function is $p(x)=\eta e^{\eta x}\mathbb{I}[x>0]$. 

\vspace{-0.1cm}
\section{FPL for Known Transition}
\vspace{-0.1cm}
FPL is originally proposed as an algorithm for efficient online linear optimization~\cite{kalai2003efficient}. In the adversarial MDP problem where the transition is known, FPL can be applied directly~\cite{even2009online}. In the beginning, we sample $r_0$, a random reward function i.i.d. from $\text{Exp}(\eta)$. Then, in episode $t$, we compute $\pi_t$ as the optimal policy on $r_{0:t-1}$, and play $\pi_t$.

\begin{algorithm}
    \caption{FPL for Adversarial MDP~\cite{even2009online}}
        \begin{algorithmic}
        \STATE Sample $r_0\in \R^{SAH}$ i.i.d. from $\text{Exp}(\eta)$
        \FOR{$t=1,\cdots,T$}
            \FOR{$h=H,\cdots,1$}
                \STATE $Q_h(s,a)\gets r_{0:t-1}(s,a,h)+\pP V_{h+1}(s,a)$, $\forall s,a$
                \STATE $V_h(s)\gets \max_a Q_h(s,a)$
                \STATE $\pi_t(s,h)\gets\arg\max_a Q_h(s,a)$
            \ENDFOR
            \STATE Play $\pi_t$ in this episode, observe $r_t$
        \ENDFOR
        \end{algorithmic}
\end{algorithm}

The original analysis in~\cite{even2009online} gives an $\OO(H\sqrt{SAT})$ regret bound, which has polynomial dependence on the number of states and actions. Our contribution is a refined analysis of the same algorithm, improving the dependence on $S$ and $A$ to $\sqrt{\log(SA)}$, which is optimal.

\begin{theorem}
The regret of Algorithm 1 is bounded by
$$\E[\Reg(T)]\le \OO\left(H^2\sqrt{\log(SA)T}\right).$$
\end{theorem}

\vspace{-0.4cm}

The proof for the theorem comes in two parts. First, as in the original analysis~\cite{even2009online}, we have a lemma commonly referred to as the ``Be-the-leader lemma'' in literature, which says that if we allow the algorithm to peek one step ahead, the regret compared to the best policy in hindsight would be small.
\begin{lemma}
\label{lem:betheleader}
$\E\left[\sum_{t=1}^T V^{\pi_{t+1}}(r_t)\right] \ge \opt-\frac{H+H\ln (SA)}{\eta}.$
\end{lemma}
\vspace{-0.1cm}
The second step is to show that peeking one step into the future does not make a large difference, since $r_0$ introduces enough randomness to ``blur'' the difference. This is also the key step where we improve the original analysis. In \cite{even2009online}, this is shown by bounding the ratio between the density function of $r_{0:t}$ and $r_{0:t-1}$, which is of the order $\exp(\eta\Vert r_t\Vert_1)$. Since $r_t$ is $SAH$ dimensional, this leads to a suboptimal bound of 
$\E\left[V^{\pi_{t+1}}(s,r_t)\right]\le e^{\eta SAH}\E\left[V^{\pi_{t}}(s,r_t)\right]$.

Our key observation is that, we are only interested in the optimal policy computed on $r_{0:t-1}$ and $r_{0:t}$. The optimal policy can be computed using value iteration, which is a structured optimization process. By showing that value iteration is ``stable'', we can remove the dependence on $SA$. In particular, we show that
\begin{lemma}
\label{lem:stable}
$\E\left[V^{\pi_{t+1}}(s,r_t)\right]\le e^{\eta H^2}\E\left[V^{\pi_{t}}(s,r_t)\right]$.
\end{lemma}
\vspace{-0.1cm}
We now give a sketch proof of Lemma~\ref{lem:stable}.

For simplicity, we use $r_0(s,-a,h)$ to denote the set of random variables $\{r_0(s,a',h):a'\neq a, a'\in\mathcal{A}\}$, and $r_0(s,\cdot,h)$ the set of random variables $\{r_0(s,a,h):a\in\mathcal{A}\}$. Observe that $Q_h(s,a')$ and $V_{h+1}^{\pi_t}(s,r_{0:t-1})$ does not depend on $r_0(s,a,h)$ for $a'\neq a$.

Since $\pi_t$ is the optimal policy on $r_{0:t-1}$, $\pi_t(s,h)=\argmax_{a}\left\{Q_h(s,a)\right\}$. Thus, $\pi_t(s,h)=a$ is equivalent to the event that $r_0(s,a,h)>\max_{a'\neq a}Q_h(s,a')-r_{1:t-1}(s,a,h)-\pP V_{h+1}^{\pi_t}(s,a)$. Let us compare this event with the counterpart for $\pi_{t+1}$, which is $r_0(s,a,h)>\max_{a'\neq a}Q_h(s,a')-r_{1:t}(s,a,h)-\pP V_{h+1}^{\pi_{t+1}}(s,a)$. We can see that if we fix $r_0(s,-a,h)$ and $r_0(s,a,h')$ for $h+1\le h'\le H$, on the left hand side we have the same exponentially distributed random variable, and on the right hand side we have two constants that differ by at most $H-h+1$. Consequently,
\begin{align*}
    e^{-\eta (H-h+1)}\le \frac{\Pr\left[\pi_t(s,h)=a\right]}{\Pr\left[\pi_{t+1}(s,h)=a\right]} \le e^{\eta (H-h+1)}.
\end{align*}
This crucial fact suggests that the resulting policy of value iteration is ``stable''. As a result, under $\pi_t$ and $\pi_{t+1}$, the probability of experiencing a trajectory $s_1,a_1,\cdots,s_H,a_H$ is also close. Specifically,
\begin{align*}
\frac{\Pr_{\pi_t}\left[s_1,a_1,\cdots,s_H,a_H\right]}{\Pr_{\pi_{t+1}}\left[s_1,a_1,\cdots,s_H,a_H\right]}
&=\prod_{h=1}^H\frac{\Pr\left[\pi_t(s_h,h)=a_h\right]}{\Pr\left[\pi_{t+1}(s_h,h)=a_h\right]}\\
&\in \left[e^{-\eta H^2},e^{\eta H^2}\right].
\end{align*}
Since the total obtained reward is a function of the experienced trajectory, it naturally follows that
\begin{align*}
    \E\left[V^{\pi_{t}}(s,r_t)\right]\ge e^{-\eta H^2}\E\left[V^{\pi_{t+1}}(s,r_t)\right].
\end{align*}
We now see how Lemma~\ref{lem:stable} leads to the improved regret bound. By combining Lemma~\ref{lem:stable} with Lemma~\ref{lem:betheleader}, we get
\begin{equation*}
    \begin{aligned}
        \E\left[\sum_{t=1}^T V^{\pi_t}(r_t)\right] &\ge  e^{-\eta H^2}\E\left[\sum_{t=1}^T V^{\pi_{t+1}}(r_t)\right]\\
    &\ge e^{-\eta H^2}\left(\opt - \frac{H+H\ln(SA)}{\eta}\right)\\
    &\ge \opt - \eta H^2 \opt - \frac{H+H\ln(SA)}{\eta}\\
    &\ge \opt - \eta H^3T - \frac{H+H\ln(SA)}{\eta}.
    \end{aligned}
\end{equation*}
By choosing $\eta=\sqrt{\frac{1+\ln(SA)}{H^2T}}$, we get
\begin{equation*}
    \begin{aligned}
    \Reg(T)&=\opt- \E\left[\sum_{t=1}^T V^{\pi_t}(r_t)\right]\\
    &\le \eta H^3T+\frac{H+H\ln(SA)}{\eta}\\
    &\le 2H^2\sqrt{(1+\ln(SA))T},
\end{aligned}
\end{equation*}
which proves Theorem~1.

It is not hard to encode a expert problem with $SA$ experts and reward scale $[0,H]$ as an adversarial MDP problem~\cite{zimin2013online}. This gives a regret lower bound of $\Omega(H\sqrt{\ln(SA)T})$. Our bound for FPL matches the lower bound in terms of the dependence on $S$, $A$ and $T$, but is not tight in the dependence on $H$.

O-REPS, a mirror-descent based algorithm~\cite{zimin2013online}, achieves a $\OO(H\sqrt{\ln(SA)T})$ regret bound in this setting, which matches the lower bound. However, O-REPS runs much slower than FPL. In particular, the runtime for FPL is $\OO(S^2A)$ per episode. In contrast, in each episode O-REPS needs to solve a convex optimization problem with $S$ variables, where the objective function requires $\Omega(S^2A)$ time to evaluate once (either the function value or the gradient). Thus, if a standard first-order method is used to solve this optimization problem, the running time would be at least $\Omega(S^2A\times \text{(Gradient Complexity)})$. Clearly, FPL is more computationally efficient.

We would also like to remark that when $H=S=1$, the adversarial MDP problem is exactly the experts problem~\cite{cesa2006prediction}. Thus the proof of Theorem~1 also gives an alternative proof of the regret bound of FPL applied to experts problem (see Section B of the appendix).

\section{FPL for Unknown Transition}
\vspace{-0.1cm}
In the case where the transition of the MDP is unknown but fixed, \cite{neu2012adversarial} proposes the FPOP algorithm, which combines FPL with UCRL~\cite{jaksch2010near} and achieves a regret bound of $\tilde{\OO}\left(SAH\sqrt{HT}\right)$. By leveraging our observation about the stability of value iteration, we can improve the regret bound without changing the algorithm.

First, let us introduce some additional notations for clarity. We use $W(r,P,\pi,s)$ to denote the value function of policy $\pi$ under the MDP $(r,P)$ evaluated at state $s$. We use $N_t(s,a)$ to denote the number of times that a state-action pair $(s,a)$ is visited up to episode $t$, and $N_t(s,a,s')$ to denote the number of times that after visiting $(s,a)$, the next state is $s'$. We use $\bar{P}_t$ to denote the empirical estimate of the transition function. In particular, $\bar{P}_t(s'|s,a):=\frac{N_t(s,a,s')}{\max\{1,N_t(s,a)\}}$.

We now state the FPOP algorithm for completeness (Algorithm 2). Here the maximization in line $4$ is implemented using extended value iteration~\cite{jaksch2010near} (Algorithm 3).

\begin{algorithm}[t]
        \begin{algorithmic}[5]
            \STATE Initialize $i(1)=1$, $n_1(s,a)=0$, $N(s,a)=0$ and $M_1(s,a)=0$ for all $(s,a,h)$; initialize $\mathcal{P}_t$ as the set of all possible transitions
            \STATE Sample i.i.d. $r_0(s,a,h)\sim\text{Exp}(\eta)$ for all $(s,a,h)$
            \FOR{$t=1,\cdots,T$}
                \STATE Choose $(\pi_t,\tilde{P}_t)\gets \argmax_{\pi,P\in \mathcal{P}_t} W(r_{0:t-1},P,\pi)$
                \FOR{$h=1,\cdots,H$}
                    \STATE Observe state $s^t_h$, take action $a^t_h=\pi_t(s^t_h)$
                    \STATE $n_{i(t)}(s^t_{h},a^t_h)\gets n_{i(t)}(s^t_h,a^t_h)+1$
                \ENDFOR
                \STATE Update $N_t(s,a)$, $N_t(s,a,s')$ and $\bar{P}_t(s'|s,a)$ accordingly
                \IF{$n_{i(t)}(s,a)\ge N_t(s,a)$ for some $(s,a)$, start new epoch}
                    \STATE $i(t+1)=i(t)+1$; Compute $\bar{P}$, the empirical transition function
                    \STATE Update $\mathcal{P}$ as
                    \begin{align*}\mathcal{P}_{i(t)+1}\gets \mathcal{P}_{i(t)}\cap& \bigg\{P:\Vert P(\cdot|s,a)-\bar{P}(\cdot|s,a)\Vert_1 
                    \\
                    &\le  \sqrt{\frac{2S\ln\frac{SAT}{\delta}}{\max\{1,N_t(s,a)\}}},~~\forall s,a\bigg\}
                    \end{align*}
                    \STATE Reset $n_{i(t+1)}(s,a)\gets 0$; resample $r_0\sim\text{Exp}(\eta)$ 
                \ELSE
                    \STATE $i(t+1)=i(t)$
                \ENDIF
            \ENDFOR
        \end{algorithmic}
        \caption{FPOP Algorithm for Adversarial MDP with Unknown Transition~\cite{neu2012adversarial}}
\end{algorithm}


\begin{algorithm}
    \begin{algorithmic}
        \STATE Input: value function $r$, empirical estimate $\bar{P}$, counters $N(s,a)$
        \STATE Compute $b(s,a)\gets\sqrt{\frac{2S\ln\frac{SAT}{\delta}}{\max\{1,N(s,a)\}}}$
        \STATE Initialize $w_{H+1}(s)=0$ for all $s$
        \FOR{$h=H,\cdots,1$}
            \STATE Sort states into $(s^*_1,\cdots,s^*_S)$ in descending order of $w_{h+1}(\cdot)$
            \FOR{$s\in \mathcal{S}$, $a\in \mathcal{A}$}
                \STATE $P^*(s^*_1|s,a)\gets \min\{\bar{P}(s^*_1|s,a)+b(s,a)/2,1\}$
                \STATE $P^*(s^*_i|s,a)\gets \bar{P}(s^*_i|s,a)$ for $k=2,\cdots,S$
                \STATE $j\gets S$
                \WHILE{$\sum_i P^*(s^*_i|s,a)>1$}
                    \STATE $P^*(s^*_i|s,a)=\max\{0,1-\sum_{i\neq j}P^*(s^*_i|s,a)\}$
                    \STATE $j\gets j-1$
                \ENDWHILE
            \ENDFOR
            \FOR{$s\in\mathcal{S}$}
                \STATE $w_{h}(s)\gets \max_a \{r(s,a)+\sum_{s'}P^*(s'|s,a)w_{h+1}(s')\}$
                \STATE $\pi(s,h)\gets \argmax_a \{r(s,a)+\sum_{s'}P^*(s'|s,a)w_{h+1}(s')\}$
            \ENDFOR
        \ENDFOR
    \end{algorithmic}
    \caption{Extended Value Iteration~\cite{jaksch2010near}}
\end{algorithm}

We proceed to give a quick overview of the original anlaysis in~\cite{neu2012adversarial}.
Let $\tilde{v}_t:=W(r_t,\pi_t,\tilde{P}_t,s^t_{1})$ be the value of algorithm's policy on the optimistic transition; let $$(\hat{\pi}_t,\hat{P}_t)\gets \argmax_{\pi,P\in \mathcal{P}_{t}}\left\{W(r_{0:t},\pi,P)\right\},$$
and let $\hat{v}_t:=W(r_t,\hat{\pi}_t,\hat{P}_t, s^t_{1})$ be the optimistic value of the ``one-step lookahead'' policy.

Similar to the known transition case,~\cite{neu2012adversarial} also shows that allowing the algorithm to peek one step into the future doesn't make much difference by bounding the ratio between the density of $r_{0:t-1}$ and $r_{0:t}$. In particular, they prove the following lemma.

\begin{lemma}[Lemma 3 in~\cite{neu2012adversarial}]
\label{lem:stableucrl}
\begin{equation*}
\E\left[\sum_{t=1}^T \hat{v}_t\right] \le \E\left[\sum_{t=1}^T \tilde{v}_t\right]+(e-1)\eta SAH \cdot HT.
\end{equation*}
\end{lemma}
Next, $\E[\sum_{t=1}^T \tilde{v}_t]$ is bounded as in the analysis of UCRL.
\begin{lemma}[Lemma 5 in~\cite{neu2012adversarial}]
\label{lem:ucrl}
Assume that $T\ge HSA$ and set $\delta=1/(HT)$. Then
\begin{equation*}
    \E\left[\sum_{t=1}^T \tilde{v}_t\right]\le \E\left[\algo\right]+\tilde{\OO}\left(H^2S\sqrt{AT}\right).
\end{equation*}
\end{lemma}
Again, by observing that extended value iteration is a structured optimization process, we can show that
$$\frac{\Pr[\hat\pi_t(s,h)=a]}{\Pr[\pi_t(s,h)=a]}\in\left[e^{-\eta H},e^{\eta H}\right].$$
Thus, by focusing on the induced policy rather than the distribution of the reward, we can obtain a better bound to supersede Lemma~\ref{lem:stableucrl} and Lemma~\ref{lem:ucrl}.
\begin{lemma}
\label{lem:stablenew}
Suppose that $\eta\le H^{-2}$, then $$\E\left[\sum_{t=1}^T \hat{v}_t\right] \le \algo + (e-1)\eta H^2\cdot HT+\tilde{O}\left(H^2S\sqrt{AT}\right).$$
\end{lemma}
This will give a drop-in improvement on the regret bound of FPOP. In particular, we improved the dependence on $A$ to $\sqrt{A}$.
\begin{theorem}
The regret of Algorithm 2 is bounded by $\tilde{O}\left(H^2S\sqrt{AT}\right)$.
\end{theorem}
The recent work of~\cite{rosenberg2019online} also achieves the same $\tilde{O}\Bigl(H^2S\sqrt{AT}\Bigr)$ regret bound, using an algorithm based on O-REPS and UCRL. Including O-REPS as a subroutine, their algorithm also needs to solve a convex optimization problem each episode, where the objective function requires $\Omega(S^2A)$ time to evaluate the function value or the gradient. In comparison, the computational cost of FPOP is $O(S^2A)$ per episode, which is much more efficient.

\section*{Acknowledgements}
This works was done as the course project for the 2019 Fall Stochastic Network Optimization Theory course at Tsinghua University, instructed by Longbo Huang. The authors thank Longbo Huang and Tiancheng Yu for helpful discussions.

\bibliography{example_paper}
\bibliographystyle{icml2020}


\clearpage

\newpage

\onecolumn



\appendix
\section{Some Basic Facts}
\begin{fact}\label{fact:lip}
Suppose that random variable $X\sim \text{Exp}(\eta)$. Denote the c.d.f. of $X$ by $F(x)$. Then $\ln(1-F(x))$ is $\eta$-Lipschitz. In other words, $F(x+\Delta)\le e^{\eta\Delta}F(x)$ for any $x$ and $\Delta\ge 0$.
\end{fact}
\begin{proof}[Proof of Fact 1]
$1-F(x)=\min\{1,e^{-\eta x}\}$. Thus $\ln(1-F(x))=\min\{0,-\eta x\}$. This is obviously $\eta$-Lipschitz.
\end{proof}

The following fact is about the maximum of independent exponential random variables (see, \emph{e.g.}, \citet[Corollary 4.5]{cesa2006prediction}). We state the proof for completeness.
\begin{fact}\label{fact:sup-of-exp}
Suppose $X_1,\cdots,X_m$ are i.i.d. random variables drawn from $\text{Exp}(\eta)$, then
\begin{align*}
    \E\left[\max_{1\le i\le m} X_i\right]\le \frac{1+\ln m}{\eta}.
\end{align*}
\end{fact}
\begin{proof}[Proof of Fact 2]
\begin{flalign*}
&& \E\left[\max_i X_i\right]&=\int_{0}^{\infty}\Pr\left[\max_{1\le i\le m} X_i>t\right]{\rm d}t &&\\
&& &\le a+\int_{a}^{\infty}\Pr\left[\max_{1\le i\le m} X_i>t\right]{\rm d}t &&\\
&& &\le a+\int_{a}^{\infty}m\Pr\left[X_1>t\right]{\rm d}t &&\text{(Union bound)}\\
&& &= a + \frac{m}{\eta}e^{-\eta a}.
\end{flalign*}
Choosing $a=\frac{\ln m}{\eta}$ proves the statement.
\end{proof}

\section{FPL for Experts Problem}
Prediction with expert advice~\cite{cesa1997use} is a classic problem in online learning. Here, there are $n$ experts. In round $t$, each expert suffers a cost in $[0,1]$. The cost of the $n$ experts is called a loss vector $l_t\in\R^n$. The agent needs to choose to follow an expert in round $t$ before $l_t$ is revealed. The goal of the agent is to minimize regret, the gap between the algorithm's cost and that of the best expert in $T$ rounds. Using the formulation of this project, the expert problem is a special case of adversarial MDP with $H=1$, $S=1$ and $A=n$.

In many previous texts about the analysis of FPL for expert problems, a problem similar to the ``$\eta SA$-stable'' problem in~\citet{even2009online} exists as well: since $\Vert l_t\Vert_1$ can be as large as $n$, the stability argument based on density ratio leads to a suboptimal $O\left(\sqrt{Tn\log(n)}\right)$ regret. To solve that matter, a clever trick is needed to argue that assuming $\Vert l_t\Vert_1\le 1$ is not without loss of generality (see footnote 8 in~\citet{kalai2003efficient} or Sec 1.7 in~\citet{kleinberg2007}). Specifically, given a loss vector $l_t=(c_1,\cdots,c_n)$, imagine that instead of $l_t$, the following sequence of loss vectors are given to the algorithm:
\begin{align*}
    &(c_1,0,\cdots,0)\\
    &(0,c_2,\cdots,0)\\
    &\quad\cdots\\
    &(0,0,\cdots,c_n).
\end{align*}
It is then argued that after this decomposition, $\opt$ doesn't change while for FPL, $\algo$ can only increase (thus regret can only increase).

However, our observation for the adversarial MDP problem in fact provides an alternative to this clever trick. Indeed, when one plug in $H=1$, $S=1$ and $A=n$, the $O(H^2\sqrt{\ln(SA)T})$ regret bound becomes $O(\sqrt{\ln(n)T})$, which is already optimal~\citep{cesa2006prediction}.

\section{Proof of Lemma 1}
\begin{lemma}
$$\E\left[\sum_{t=1}^T V^{\pi_{t+1}}(r_t)\right] \ge \opt-\frac{H+H\ln (SA)}{\eta}.$$
\end{lemma}
\begin{proof}
Since $\pi_t$ is the greedy policy computed on $r_{0:t-1}$,
\begin{align*}
    V^{\pi_t}(r_{0:t-1})\ge V^{\pi_{t+1}}(r_{0:t-1}).
\end{align*}
Rearranging the inequality, we get
\begin{align*}
    V^{\pi_{t+1}}(r_t)\ge V^{\pi_{t+1}}(r_{0:t})-V^{\pi_t}(r_{0:t-1}).
\end{align*}
Summing from $t=0$ to $T$, we get
\begin{align*}
    \sum_{t=0}^T V^{\pi_{t+1}}(r_t) \ge V^{\pi_{T+1}}(r_{0:T})=V^*(r_{0:T})\ge V^*(r_{1:T})=\opt.
\end{align*}
It follows that
\begin{align*}
     \sum_{t=1}^T V^{\pi_{t+1}}(r_t)&\ge \opt-V^{\pi_1}(r_0).
\end{align*}
Thus
\begin{align*}
    \E\left[\sum_{t=1}^T V^{\pi_{t+1}}(r_t)\right] &\ge \opt-\E\left[V^{\pi_1}(r_0)\right]\\
    &\ge \opt-\sum_{h=1}^{H}\E\left[\sup_{s,a}r_0(s,a,h)\right]\\
    &\ge \opt-\frac{H+H\ln (SA)}{\eta}.\tag{By Fact~\ref{fact:sup-of-exp}}
\end{align*}
\end{proof}

\section{Proof of Lemma 2}
\begin{lemma}
$$\E\left[V^{\pi_{t+1}}(s,r_t)\right]\le e^{\eta H^2}\E\left[V^{\pi_{t}}(s,r_t)\right]$$
\end{lemma}
\begin{proof}
This result follows from the fact that $\pi_t$ and $\pi_{t+1}$ are close, which in turn follows from the stability of value iteration.

For shorthand, we use $r_0(s,-a,h)$ to denote the set of random variables $\{r_0(s,a',h):a'\neq a\},$ and $r_0(s,\cdot,h)$ the set of random variables $\{r_0(s,a,h):a\in [A]\}.$  Observe that $Q_h(s,a')$ and $V_{h+1}^{\pi_t}(s,a)$ does not depend on $r_0(s,a,h)$ for $a'\neq a$. Let $\mathcal{E}$ be the event that $r_0(s,a,h)>\max_{a'\neq a}Q_h(s,a')-r_{1:t-1}(s,\cdot,h)-\pP V_{h+1}^{\pi_t}(s,a)$. It follows that
\begin{align*}
    &\Pr\left[\pi_t(s,h)=a\right] \\
    =& \E_{r_0(s,-a,h),r_0(s,\cdot,h+1:H)}\left[\Pr\left[\left.\mathcal{E}\right|r_0(s,-a,h),r_0(s,\cdot,h+1:H)\right]\right]\\
    =& \E_{r_0(s,-a,h),r_0(s,a,h+1:H)}\left[1-F\left(\max_{a'\neq a}Q_h(s,a')-r_{1:t-1}(s,a,h)-\pP V^{\pi_t}_{h+1}(s,a)\right)\right].
\end{align*}
Obviously $\left|V^{\pi_{t+1}}_{h+1}(s',r_{0:t})-V^{\pi_t}_{h+1}(s',r_{0:t})\right|\le H-h$. Thus for fixed $r_0(s,h,-a)$ and $r_0(s,a,h+1:H)$, both $Q_h(s,a')$ and $\pP V^{\pi_t}_{h+1}(s,a)$ can only change by $H-h$. Since $\ln(1-F(x))$ is $\eta$-Lipschitz,
\begin{align*}
    e^{-\eta(H-h+1)}&\le \frac{1-F\left(\max_{a'\neq a}Q_h(s,a')-r_{1:t-1}(s,a,h)-\pP V^{\pi_t}_{h+1}(s,a)\right)}{1-F\left(\max_{a'\neq a}Q_h(s,a')-r_{1:t}(s,a,h)-\pP V^{\pi_{t+1}}_{h+1}(s,a)\right)} \le e^{\eta (H-h+1)}.
\end{align*}
In other words,
\begin{align*}
    e^{-\eta (H-h+1)}\le \frac{\Pr\left[\pi_t(s,h)=a\right]}{\Pr\left[\pi_{t+1}(s,h)=a\right]} \le e^{\eta (H-h+1)}.
\end{align*}
It follows that for any trajectory $s_1,a_1,\cdots,s_H,a_H$,
\begin{align*}
&\frac{\Pr_{\pi_t}\left[s_1,a_1,\cdots,s_H,a_H\right]}{\Pr_{\pi_{t+1}}\left[s_1,a_1,\cdots,s_H,a_H\right]}\\
=&\frac{\Pr[s_1]\cdot\Pr_{\pi_t}\left[a_1|s_1\right]\cdot \Pr \left[s_2|s_1,a_1\right]\cdot \cdots \Pr\left[s_H|s_{H-1},a_{H-1}\right]\cdot\Pr_{\pi_{t}}\left[a_H|s_H\right] }{\Pr[s_1]\cdot\Pr_{\pi_{t+1}}\left[a_1|s_1\right]\cdot \Pr \left[s_2|s_1,a_1\right]\cdot \cdots \Pr\left[s_H|s_{H-1},a_{H-1}\right]\cdot\Pr_{\pi_{t+1}}\left[a_H|s_H\right]}\\
=&\prod_{h=1}^H\frac{\Pr_{\pi_t}\left[a_h|s_h\right]}{\Pr_{\pi_{t+1}}\left[a_h|s_h\right]}\in \left[e^{-\eta H^2},e^{\eta H^2}\right].
\end{align*}
Thus
\begin{align*}
\E\left[V^{\pi_{t}}(s,r_t)\right]&=\sum_{\text{all trajectories}}\Pr_{\pi_t}\left[s_1,a_1,\cdots,s_H,a_H\right]\cdot \left(\sum_{h=1}^H r_t(s_h,a_H)\right)\\
&\ge e^{-\eta H^2}\sum_{\text{all trajectories}}\Pr_{\pi_{t+1}}\left[s_1,a_1,\cdots,s_H,a_H\right]\cdot  \left(\sum_{h=1}^H r_t(s_h,a_H)\right)\\
&= e^{-\eta H^2}\E\left[V^{\pi_{t+1}}(s,r_t)\right].
\end{align*}
\end{proof}

\section{Proof of Theorem 1}
\begin{proof}[Proof of Theorem 1]
By combining lemma 1 and lemma 2, we get
\begin{align*}
    \E\left[\sum_{t=1}^T V^{\pi_t}(r_t)\right] &\ge  e^{-\eta H^2}\E\left[\sum_{t=1}^T V^{\pi_{t+1}}(r_t)\right]\\
    &\ge e^{-\eta H^2}\left(\opt - \frac{H+H\ln(SA)}{\eta}\right)\\
    &\ge \opt - \eta H^2 \opt - \frac{H+H\ln(SA)}{\eta}\\
    &\ge \opt - \eta H^3T - \frac{H+H\ln(SA)}{\eta}.
\end{align*}
By choosing $\eta=\sqrt{\frac{1+\ln(SA)}{H^2T}}$, we get
\begin{equation}
    \Reg(T)=\opt-\algo \le \eta H^3T+\frac{H+H\ln(SA)}{\eta}\le 2H^2\sqrt{(1+\ln(SA))T}.
\end{equation}
\end{proof}

\section{Proof of Lemma 5}
\addtocounter{lemma}{2}
\begin{lemma}
Suppose that $\eta\le H^{-2}$, then
\label{lem:stablenew}
\begin{equation*}
\E\left[\sum_{t=1}^T \hat{v}_t\right] \le \algo + (e-1)\eta H^2\cdot HT+\tilde{O}\left(H^2S\sqrt{AT}\right).
\end{equation*}
\end{lemma}
\begin{proof}
Recall that
$$\tilde{v}_t=W(r_t,\pi_t,\tilde{P}_t,s_{t,1}),\quad \hat{v}_t=W(r_t,\hat{\pi}_t,\hat{P}_t,s_{t,1}).$$
Let us also define $\bar{v}_t:=W(r_t,\pi_t,\hat{P}_t,s_{t,1})$.

Now, consider the extended value iteration process. First, observe that $w_{h+1}(s')$ is determined by $\bar{P}$ and $r_0(s,a,h+1:H)$. Let us use $Q_h(s,a)$ as a shorthand for $r_{1:t-1}(s,a)+\sum_{s'}P^*(s'|s,a)w_{h+1}(s')$, where $w$ is computed on $r_{0:t-1}$; similarly let $\hat{Q}_h(s,a)$ as a shorthand for $r_{1:t}(s,a)+\sum_{s'}P^*(s'|s,a)w_{h+1}(s')$, where $w$ is computed on $r_{0:t}$. We can write $\Pr\left[\pi_t(s,h)=a\right]$ as
\begin{align*}
\E_{r_0(s,-a,h),r_0(s,\cdot,h+1:H)}\left[1-F\left(\max_{a'\neq a}(Q_h(s,a')+r_0(s,a',h))-Q_h(s,a)\right)\right].
\end{align*}
Similarly
\begin{align*}
\Pr\left[\hat{\pi}_t(s,h)=a\right]=\E_{r_0(s,-a,h),r_0(s,\cdot,h+1:H)}\left[1-F\left(\max_{a'\neq a}(\hat{Q}_h(s,a')+r_0(s,a',h))-\hat{Q}_h(s,a)\right)\right].
\end{align*}
Observe that $0\le \hat{Q}_h(s,a)-{Q}_h(s,a)\le H-h+1$. It follows from Fact 1 that
\begin{align*}
    \frac{\Pr\left[\hat{\pi}_t(s,h)=a\right]}{\Pr\left[\pi_t(s,h)=a\right]}\in\left[e^{-\eta H},~e^{\eta H}\right].
\end{align*}
Using the same argument for value iteration, we can show that
\begin{align*}
    \E\left[\hat{v}_t\right]&=\E\left[W(r_t,\hat{\pi}_t,\hat{P}_t,s_{t,1})\right]\\
    &\le e^{\eta H^2}\E\left[W(r_t,\pi_t,\hat{P}_t,s_{t,1})\right]\\
    &\le \E\left[\bar{v}_t\right]+(e-1)\eta H^3,
\end{align*}
where we used the identity that $e^x\le 1+(e-1)(x-1)$ for $x\in[0,1]$.

In the proof of Lemma 5~\cite{jaksch2010near}, the only property of $\hat{P}_t$ that is used is $\hat{P}_t\in\mathcal{P}_t$. Since $\tilde{P}_t\in\mathcal{P}_t$ as well, from the same proof it follows that
\begin{align*}
\E\left[\sum_{t=1}^T \bar{v}_t\right]&\le \algo+HS\sqrt{2T\ln\frac{H}{\delta}}+2\delta HT+(\sqrt{2}+1)H^2S\sqrt{TA\ln\frac{SAT}{\delta}}\\
&\le \algo + \tilde{O}\left(H^2S\sqrt{AT}\right).
\end{align*}

\end{proof}

\section{Proof of Theorem 2}
\begin{proof}
First, without loss of generality, assume that $T>H^2SA$~\footnote{Otherwise, $HT\le H^2\sqrt{SAT}$, so the regret bound holds trivially.}.

Let us state a lemma from the original FPOP analysis~\cite{neu2012adversarial}, which has a similar flavor to the ``be-the-leader'' lemma.
\begin{lemma}[Lemma 2~\cite{neu2012adversarial}]
\label{lem:bpl}
\begin{equation*}
    \opt \le \sum_{t=1}^T \E[\hat{v}_t]+\delta HT+SA\log\left(\frac{8T}{SA}\right)\frac{H\ln(SA)+H}{\eta}.
\end{equation*}
\end{lemma}
Let us choose $\eta=\sqrt{\frac{SA}{H^2T}}$ and $\delta=1/(HT)$. We can see that $\eta<1/(H^2)$. Then by Lemma~\ref{lem:stablenew} and~\ref{lem:bpl},
\begin{align*}
    \opt &\le \algo+\tilde{O}\left(H^2S\sqrt{AT}\right)+SA\log\left(\frac{8T}{SA}\right)\frac{H\ln(SA)+H}{\eta}+(e-1)\eta H^3T\\
    &= \algo+\tilde{O}\left(H^2S\sqrt{AT}\right)+H^2\sqrt{SAT}\cdot\left[\log\left(\frac{8T}{SA}\right)(\ln(SA)+1)+e-1\right]\\
    &= \algo+\tilde{O}\left(H^2S\sqrt{AT}\right).
\end{align*}
In other words, $ \Reg(T) \le \tilde{O}\left(H^2S\sqrt{AT}\right).$
\end{proof}

\end{document}